\documentclass[letterpaper]{article} % DO NOT CHANGE THIS
\usepackage[preprint]{aaai2027}
\usepackage[hyphens]{url} % DO NOT CHANGE THIS
\usepackage{graphicx} % DO NOT CHANGE THIS
\usepackage{natbib} % DO NOT CHANGE THIS
\usepackage{caption} % DO NOT CHANGE THIS
\usepackage{cuted}
\usepackage{booktabs}
\usepackage{amsmath,amssymb,amsthm}
\usepackage{mathtools}
\usepackage{cleveref}
\usepackage{tikz}
\usetikzlibrary{arrows.meta,backgrounds,positioning}
\usepackage{pgfplots}
\pgfplotsset{compat=1.17}
\usepackage{enumitem}

\newtheorem{theorem}{Theorem}

\newtheorem{corollary}[theorem]{Corollary}

\theoremstyle{definition}
\newtheorem{definition}[theorem]{Definition}
\newtheorem{remark}[theorem]{Remark}
\newtheorem{example}[theorem]{Example}
\newtheorem{assumption}[theorem]{Assumption}

\newcommand{\E}{\mathbb{E}}
\newcommand{\R}{\mathbb{R}}
\newcommand{\pQ}{p_{\mathcal{Q}}}

\newcommand{\piold}{\pi_{\theta_{\mathrm{old}}}}
\newcommand{\pith}{\pi_\theta}
\newcommand{\Atilde}{\tilde{A}}

\newcommand{\bq}{\mathbf{q}}
\newcommand{\bo}{\mathbf{o}}
\newcommand{\bR}{\mathbf{R}}

\title{On the Impossibility of Unbiased and Length-Invariant Policy Optimization with Outcome Rewards}

\author{Fei Ding\textsuperscript{1}\thanks{Correspondence to: Fei Ding (\texttt{dignfei@gmail.com}).}, Yongkang Zhang\textsuperscript{1}, Runhao Liu\textsuperscript{1}, Yuhao Liao\textsuperscript{2}, Zijian Zeng\textsuperscript{2}, Huiming Yang\textsuperscript{2}}
\affiliations{\textsuperscript{1}Alibaba Group\\
\textsuperscript{2}Tsinghua University}

\begin{document}

\maketitle

\begin{abstract}
Group Relative Policy Optimization (GRPO) is the dominant reinforcement learning algorithm for training reasoning capabilities in large language models, notably adopted by DeepSeek-R1.
The recent improvement Dr.\ GRPO (COLM 2025) identifies the response-level length bias caused by per-trajectory length normalization in GRPO and proposes removing this normalization, claiming the resulting optimizer is ``unbiased.''
We show that this claim is incomplete.
Specifically, we establish an \emph{impossibility theorem}: under the standard outcome reward + GRPO setting, no length-based weighting scheme can simultaneously achieve the following two properties. (P1)~\emph{Gradient unbiasedness}: the gradient estimator is an unbiased estimate of the true policy gradient. (P2)~\emph{Length invariance}: each trajectory's effective contribution to the gradient is independent of its token length.
GRPO approximately satisfies P2 but violates P1; Dr.\ GRPO satisfies P1 but violates P2.
We characterize the complete tradeoff spectrum via the parametric family $f_\alpha(L) = L^{\alpha - 1}$, where $\alpha = 0$ recovers GRPO, $\alpha = 1$ recovers Dr.\ GRPO, and provide quantitative analysis showing that Dr.\ GRPO's length bias can cause longer trajectories to dominate gradient updates by a factor proportional to the length ratio.
Our results reveal that neither algorithm is universally ``done right''; they occupy opposite ends of a fundamental and unavoidable tradeoff.
\end{abstract}

\begin{strip}
    \centering
    \resizebox{0.96\textwidth}{!}{%
    \begin{tikzpicture}[
        >=Stealth,
        boxstyle/.style={draw, rounded corners=3pt, minimum height=0.5cm, align=center, font=\scriptsize, inner sep=3pt},
        tokstyle/.style={minimum width=0.32cm, minimum height=0.32cm, font=\tiny, inner sep=0.5pt, rounded corners=1pt},
        tok_ok/.style={tokstyle, fill=green!15, draw=green!40},
        tok_err/.style={tokstyle, fill=red!18, draw=red!45},
        arr/.style={->, semithick, >=Stealth},
        titlefont/.style={font=\scriptsize\bfseries},
        lblfont/.style={font=\tiny},
        ]

        % Background boxes
        \begin{scope}[on background layer]
            \fill[teal!5, draw=teal!40, rounded corners=4pt, semithick]
            (-7.6, -1.35) rectangle (-1.0, 1.6);
            \fill[orange!5, draw=orange!50, rounded corners=4pt, semithick]
            (0.7, -1.35) rectangle (6.8, 1.6);
        \end{scope}

        % ============================================================
        % LEFT: GRPO
        % ============================================================
        \node[titlefont, text=teal!70!black] at (-4.3, 1.3) {GRPO ($\alpha\!=\!0$, $f\!=\!1/L$)};

        % short wrong trajectory (L=200)
        \node[lblfont, text=black!55, anchor=east] at (-6.3, 0.7) {\textit{Short wrong}};
        \node[tok_err] (s1) at (-6.2, 0.7) {};
        \node[right=2pt of s1, font=\tiny, text=red!55!black] {$\times$};
        \node[lblfont, text=black!40, anchor=west] at (-5.65, 0.7) {\textit{L\,=\,200}};
        % per-token weight
        \node[lblfont, text=teal!60!black, anchor=west] at (-4.7, 0.7) {Per-token penalty: $\frac{0.5}{200}$};

        % long wrong trajectory (L=6000)
        \node[lblfont, text=black!55, anchor=east] at (-6.3, 0.15) {\textit{Long wrong}};
        \node[tok_err] (l1) at (-6.2, 0.15) {};
        \node[tok_err, right=0.5pt of l1] (l2) {};
        \node[right=1pt of l2, font=\tiny, text=black!40] {$\cdots$};
        \node[right=2pt of l2, font=\tiny, text=red!55!black, xshift=5pt] {$\times$};
        \node[lblfont, text=black!40, anchor=west] at (-5.25, 0.15) {\textit{L\,=\,6000}};
        \node[lblfont, text=teal!60!black, anchor=west] at (-4.2, 0.15) {Per-token penalty: $\frac{0.5}{6000}$};

        % bias arrow
        \node[lblfont, text=red!60!black, anchor=west, align=left] at (-7.4, -0.5)
        {30$\times$ penalty gap $\Rightarrow$ \textbf{Model biased toward long errors}};

        % properties
        \node[boxstyle, fill=red!8, draw=red!40, text=red!65!black] at (-5.2, -1.0)
        {P1 Unbiased \textbf{$\boldsymbol{\times}$}};
        \node[boxstyle, fill=green!8, draw=green!40, text=green!55!black] at (-2.9, -1.0)
        {P2 Length-inv.\ \textbf{$\boldsymbol{\checkmark}$}};

        % ============================================================
        % CENTER: impossibility
        % ============================================================
        \node[font=\scriptsize\bfseries, text=violet!70!black, align=center] at (0, 0.5)
        {Cannot\\[-1pt]have both};
        \node[font=\Large, text=red!60!black] at (0, -0.15) {$\boldsymbol{\times}$};
        \draw[arr, violet!50, thick] (-0.6, 0.1) -- (-0.2, 0.1);
        \draw[arr, violet!50, thick] (0.6, 0.1) -- (0.2, 0.1);

        % ============================================================
        % RIGHT: Dr. GRPO
        % ============================================================
        \node[titlefont, text=orange!70!black] at (3.9, 1.3) {Dr.\ GRPO ($\alpha\!=\!1$, $f\!=\!1$)};

        % short correct trajectory (L=200)
        \node[lblfont, text=black!55, anchor=east] at (2.0, 0.7) {\textit{Short correct}};
        \node[tok_ok] (ds1) at (2.2, 0.7) {};
        \node[right=2pt of ds1, font=\tiny, text=green!50!black] {$\checkmark$};
        \node[lblfont, text=black!40, anchor=west] at (2.75, 0.7) {\textit{L\,=\,200}};
        \node[lblfont, text=orange!60!black, anchor=west] at (3.6, 0.7) {Contrib: $200\!\times\!0.5\!=\!100$};

        % long wrong trajectory (L=6000)
        \node[lblfont, text=black!55, anchor=east] at (2.0, 0.15) {\textit{Long wrong}};
        \node[tok_err] (dl1) at (2.2, 0.15) {};
        \node[tok_err, right=0.5pt of dl1] (dl2) {};
        \node[right=1pt of dl2, font=\tiny, text=black!40] {$\cdots$};
        \node[right=2pt of dl2, font=\tiny, text=red!55!black, xshift=5pt] {$\times$};
        \node[lblfont, text=black!40, anchor=west] at (3.2, 0.15) {\textit{L\,=\,6000}};
        \node[lblfont, text=orange!60!black, anchor=west] at (4.1, 0.15) {Contrib: $6000\!\times\!0.5$};

        % bias note
        \node[lblfont, text=red!60!black, anchor=west, align=left] at (1.2, -0.5)
        {30$\times$ long-error contrib $\Rightarrow$ \textbf{Long traj.\ dominates gradient}};

        % properties
        \node[boxstyle, fill=green!8, draw=green!40, text=green!55!black] at (2.8, -1.0)
        {P1 Unbiased \textbf{$\boldsymbol{\checkmark}$}};
        \node[boxstyle, fill=red!8, draw=red!40, text=red!65!black] at (5.2, -1.0)
        {P2 Length-inv.\ \textbf{$\boldsymbol{\times}$}};

    \end{tikzpicture}
    }%
    
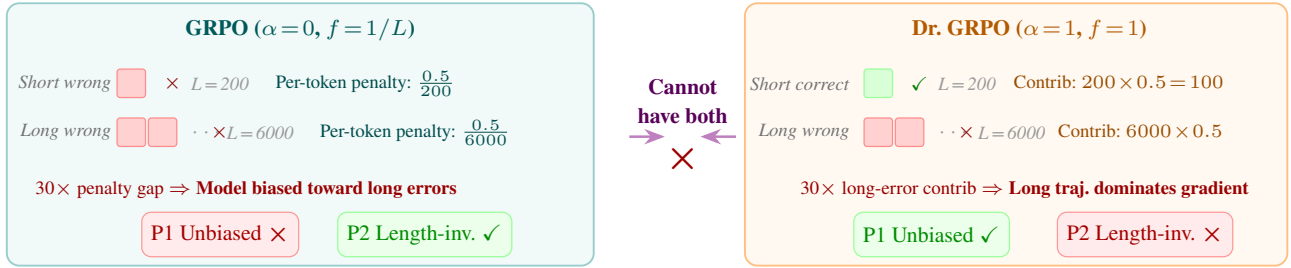
\captionof{figure}{Complete unbiasedness is impossible.}
    \label{fig:teaser}
\end{strip}

%======================================================================
\section{Introduction}
\label{sec:intro}
%======================================================================

Reinforcement learning (RL) has become a core technique for improving reasoning capabilities of large language models (LLMs).
DeepSeek-R1-Zero~\citep{deepseekai2026deepseekr1incentivizingreasoningcapability} demonstrated an important finding: without supervised fine-tuning, directly applying RL to a base LLM can elicit complex reasoning behaviors, including chain-of-thought and self-reflection.
The core algorithm of this training paradigm is Group Relative Policy Optimization (GRPO)~\citep{shao2024deepseekmathpushinglimitsmathematical}. GRPO is a critic-free RL algorithm that estimates advantages by comparing multiple responses sampled for the same prompt.

A salient empirical observation during GRPO training is the persistent growth of response length~\citep{deepseekai2026deepseekr1incentivizingreasoningcapability,zeng2025simplerl,hu2025openreasonerzeroopensourceapproach}.
\citet{liu2025understanding} critically examined this phenomenon and identified two sources of optimization bias in GRPO. The first is \emph{response-level length bias} caused by per-trajectory length normalization $\frac{1}{|\bo_i|}$, and the second is \emph{question-level difficulty bias} caused by standard deviation normalization.
They proposed Dr.\ GRPO, removing both normalization terms, claiming to restore an ``unbiased'' optimization objective.
Dr.\ GRPO has been widely adopted by the community and achieved state-of-the-art results on mathematical reasoning benchmarks at the time.

In this paper, we challenge the completeness of this claim.
We confirm that Dr.\ GRPO's gradient estimator is indeed an unbiased estimate of the policy gradient (as they rigorously proved in their Appendix~A). However, we show that removing the length normalization term $\frac{1}{|\bo_i|}$ introduces another form of bias: \emph{length bias in the optimization dynamics}. This bias causes longer trajectories to contribute disproportionately more to gradient updates.
More fundamentally, we establish the following impossibility result:

\begin{center}
\fbox{\parbox{0.94\columnwidth}{
\textbf{Main Result (Informal).} Under the outcome reward + GRPO setting, no length-based weighting scheme can simultaneously achieve gradient unbiasedness and length invariance. GRPO and Dr.\ GRPO represent the two extremes of this unavoidable tradeoff.
}}
\end{center}

Our contributions are as follows:
\begin{itemize}[nosep,leftmargin=1.5em]
    \item We formalize two desirable properties of group-based RL optimizers, namely \emph{gradient unbiasedness} (P1) and \emph{length invariance} (P2), and prove they are mutually exclusive under outcome rewards (\cref{thm:policy_class_impossibility}).
    \item We characterize the tradeoff spectrum via the parametric family $f_\alpha(L) = L^{\alpha-1}$ ($\alpha \in [0,1]$), where $\alpha=0$ corresponds to GRPO and $\alpha=1$ corresponds to Dr.\ GRPO (\cref{cor:spectrum}).
    \item We provide quantitative analysis showing that Dr.\ GRPO's length bias can be severe: at length ratio $r$, the longer trajectory captures $\frac{r}{1+r}$ of the gradient signal (\cref{cor:dr_grpo_bias}).
\end{itemize}

%======================================================================
\section{Preliminaries}
\label{sec:prelim}
%======================================================================

\paragraph{Token-level MDP.}
Language model generation is modeled as a token-level Markov Decision Process $\mathcal{M} = (\mathcal{S}, \mathcal{A}, r, \pQ)$.
At step $t$, the state $s_t = [\bq, o_1, \ldots, o_{t-1}]$ is the concatenation of the prompt and previously generated tokens.
The policy $\pith(\cdot | s_t)$ selects the next token $o_t$ from the vocabulary $\mathcal{A}$.
Generation terminates upon producing an end-of-sequence token or exhausting the token budget.
The objective is to maximize the expected return:
\begin{equation}
\label{eq:objective}
J(\pith) = \E_{\bq \sim \pQ}\left[\E_{\bo \sim \pith(\cdot|\bq)}\left[R(\bq, \bo)\right]\right],
\end{equation}
where $R(\bq, \bo) = \sum_{t=1}^{|\bo|} r(s_t, o_t)$ is the trajectory return.
Under the standard \emph{outcome reward} setting for reasoning tasks~\citep{deepseekai2026deepseekr1incentivizingreasoningcapability}, a scalar reward is assigned at the end of generation:
$R(\bq, \bo) = 1$ if $\bo$ contains the correct answer, and $0$ otherwise.

\paragraph{Policy gradient.}
The Monte Carlo policy gradient~\citep{Williams_1992,sutton2018reinforcement} of Eq.~\eqref{eq:objective} is:
\begin{equation}
\label{eq:pg}
\nabla_\theta J(\pith) = \E_{\bq, \bo \sim \pith}\left[\sum_{t=1}^{|\bo|} \nabla_\theta \log \pith(o_t | \bq, \bo_{<t}) \cdot A(o_t | \bq, \bo_{<t})\right],
\end{equation}
where $A(o_t | \bq, \bo_{<t}) = R(\bq, \bo) - B(\bq, \bo_{<t})$ is the advantage and $B$ is any baseline independent of $o_t$~\citep{sutton2018reinforcement}.
Under outcome rewards, the advantage is identical for all tokens in a trajectory since the return does not depend on $t$.

\paragraph{Group-relative baseline.}
Both GRPO and Dr.\ GRPO sample $G$ responses $\{\bo_1, \ldots, \bo_G\}$ for each prompt and use the group mean as the baseline:
$B = \mathrm{mean}(\bR)$, where $\bR = \{R(\bq, \bo_1), \ldots, R(\bq, \bo_G)\}$.
The advantage for all tokens in trajectory $\bo_i$ is:
\begin{equation}
\label{eq:advantage}
\Atilde_i = R(\bq, \bo_i) - \mathrm{mean}(\bR).
\end{equation}

\paragraph{GRPO~\citep{shao2024deepseekmathpushinglimitsmathematical}.}
GRPO maximizes the following surrogate objective (omitting the clipping mechanism as it does not affect our analysis):
\begin{equation}
\label{eq:grpo}
J_{\mathrm{GRPO}}(\theta) = \frac{1}{G}\sum_{i=1}^{G} \frac{1}{|\bo_i|}\sum_{t=1}^{|\bo_i|} \frac{\pith(o_{i,t}|\bq, \bo_{i,<t})}{\piold(o_{i,t}|\bq, \bo_{i,<t})} \cdot \frac{\Atilde_i}{\mathrm{std}(\bR)}.
\end{equation}

\paragraph{Dr.\ GRPO~\citep{liu2025understanding}.}
Dr.\ GRPO removes the per-trajectory length normalization $\frac{1}{|\bo_i|}$ and the standard deviation normalization $\mathrm{std}(\bR)$:
\begin{equation}
\label{eq:drgrpo}
J_{\mathrm{Dr.GRPO}}(\theta) = \frac{1}{G}\sum_{i=1}^{G} \sum_{t=1}^{|\bo_i|} \frac{\pith(o_{i,t}|\bq, \bo_{i,<t})}{\piold(o_{i,t}|\bq, \bo_{i,<t})} \cdot \Atilde_i.
\end{equation}
\citet{liu2025understanding} proved in their Appendix~A that the gradient of Eq.~\eqref{eq:drgrpo} recovers the unbiased Monte Carlo policy gradient with a group-relative baseline. Furthermore, the advantage $\Atilde_i$ is equivalent to REINFORCE Leave-One-Out (RLOO)~\citep{kool2019buy,ahmadian-etal-2024-back} up to a constant factor.

\begin{remark}
\label{rem:rloo}
Dr.\ GRPO's advantage estimator is equivalent to RLOO~\citep{ahmadian-etal-2024-back} up to a constant factor; see \citet{liu2025understanding} Appendix~A. Therefore our analysis also applies to RLOO, but we focus on Dr.\ GRPO since it explicitly claims to resolve the length bias issue.
\end{remark}

\paragraph{Unified framework.}
To unify the analysis of both methods, we introduce a \emph{weighted gradient estimator} parameterized by a weighting function $f: \mathbb{N} \to \R_+$:
\begin{equation}
\label{eq:unified}
\hat{g}_f = \frac{1}{G}\sum_{i=1}^{G} f(|\bo_i|) \cdot \Atilde_i \cdot \sum_{t=1}^{|\bo_i|} \nabla_\theta \log \pith(o_{i,t} | \bq, \bo_{i,<t}).
\end{equation}
GRPO corresponds to $f(L) = 1/L$ and Dr.\ GRPO corresponds to $f(L) = 1$ (both omitting the $\mathrm{std}(\bR)$ factor since it is a question-level scalar orthogonal to length bias analysis).

%======================================================================
\section{Main Result: Impossibility Theorem}
\label{sec:main}

%=========================================================================

\paragraph{Notation and setup.}
Consider the length-weighted gradient estimator
\begin{equation}
    \hat g_f
    =
    \frac{1}{G}\sum_{i=1}^{G} f(L_i)\,\tilde A_i\,S_i,
\end{equation}
where
\begin{equation}
    L_i := |\mathbf o_i|,
    \qquad
    S_i := \sum_{t=1}^{L_i}
    \nabla_\theta \log \pi_\theta(o_{i,t}\mid q,\mathbf o_{i,<t}),
\end{equation}
with the group mean baseline
\begin{equation}
    \tilde A_i
    :=
    R_i-\frac{1}{G}\sum_{j=1}^{G}R_j.
\end{equation}

In what follows, $\Pi$ denotes the policy class under consideration. We assume all expectations below exist and that within-group trajectories are conditionally i.i.d.\ given the prompt and the current policy.

\begin{assumption}[Fixed-length realizability]
    \label{ass:fixed_length_realizable}
    There exists a set of lengths $\mathcal L\subseteq \mathbb N$ such that for every $L\in\mathcal L$, the policy class $\Pi$ contains a policy $\pi^{(L)}$ under which, given the prompt, the trajectory length equals $L$ almost surely, while the token content retains non-degenerate randomness.
\end{assumption}

\begin{assumption}[Update scale functional]
    \label{ass:scale_functional}
    Fix an update scale functional
    \[
    \rho:\mathbb R^d\to\mathbb R_+,
    \]
    used to measure the magnitude of a single-trajectory score sum. We only require $\rho$ to be positively homogeneous of degree one for non-negative scalars, i.e., for all $\alpha\ge 0$ and all $v\in\mathbb R^d$,
    \begin{equation}
        \rho(\alpha v)=\alpha \rho(v).
    \end{equation}
    Typical examples include vector norms or the non-negative projection magnitude along a fixed direction.
\end{assumption}

\begin{definition}[Trajectory-level correctness P1]
    \label{def:p1_uniform}
    The estimator $\hat g_f$ satisfies \emph{trajectory-level correctness} over the policy class $\Pi$ if there exists a constant $c>0$, independent of the trajectory length distribution, such that for every policy $\pi\in\Pi$,
    \begin{equation}
        \mathbb E_{\pi}[\hat g_f]
        =
        c\,\nabla_\theta J(\pi).
    \end{equation}
\end{definition}

\begin{definition}[Length neutrality P2]
    \label{def:p2_uniform}
    Let
    \begin{equation}
        \Gamma_{\pi,\rho}(L;a)
        :=
        \mathbb E_{\pi}\!\left[\rho(S)\,\middle|\,L(\tau)=L,\ \tilde A(\tau)=a\right],
    \end{equation}
    where $a$ denotes a fixed effective training signal, i.e., a realized value of the group-relative advantage.

    The estimator $\hat g_f$ satisfies \emph{length neutrality} under the scale functional $\rho$ if for every policy $\pi\in\Pi$, every realizable length $L$, and every fixed $a$,
    \begin{equation}
        f(L)\,\Gamma_{\pi,\rho}(L;a)
    \end{equation}
    is independent of $L$.
\end{definition}

\begin{theorem}[Structural conflict at the policy-class level]
    \label{thm:policy_class_impossibility}
    Under the outcome-level reward and group mean baseline setting, consider a weight function depending only on length,
    \[
    f:\mathbb N\to\mathbb R_+.
    \]
    Suppose Assumptions~\ref{ass:fixed_length_realizable} and~\ref{ass:scale_functional} hold.

    If there exist a policy $\pi^\star\in\Pi$, an effective training signal value $a^\star$, and two distinct lengths $L_1,L_2\in\mathcal L$ such that
    \begin{equation}
        \Gamma_{\pi^\star,\rho}(L_1;a^\star)
        \neq
        \Gamma_{\pi^\star,\rho}(L_2;a^\star),
    \end{equation}
    then no such $f$ can simultaneously satisfy P1 (trajectory-level correctness) and P2 (length neutrality) over the policy class $\Pi$.
\end{theorem}

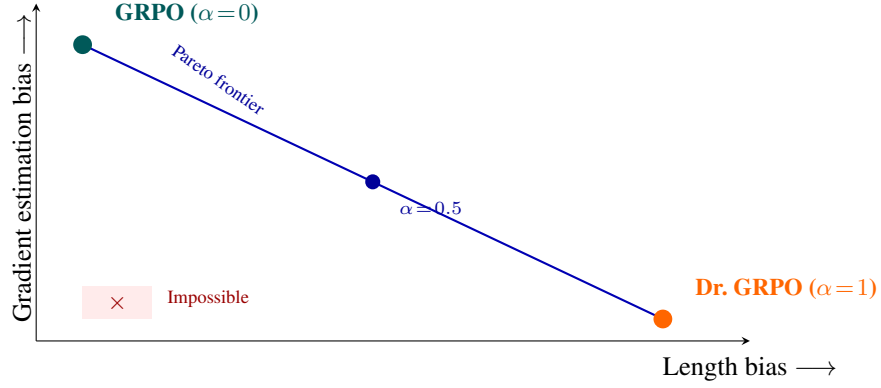
\begin{figure*}[!t]
\centering
\begin{tikzpicture}
\begin{axis}[
    width=0.62\textwidth,
    height=0.34\textwidth,
    xlabel={Length bias $\longrightarrow$},
    ylabel={Gradient estimation bias $\longrightarrow$},
    xmin=-0.08, xmax=1.15,
    ymin=-0.08, ymax=1.15,
    xtick=\empty,
    ytick=\empty,
    axis lines=left,
    clip=false,
    every axis x label/.style={at={(1.0,-0.02)}, anchor=north},
    every axis y label/.style={at={(-0.02,1.0)}, anchor=east, rotate=90},
]
% Impossible region
\fill[red!8] (0,0) rectangle (0.12,0.12);
\node[red!60!black, font=\footnotesize\bfseries] at (0.06, 0.06) {$\times$};
\node[red!60!black, font=\scriptsize, anchor=south west] at (0.13, 0.01) {Impossible};

% Pareto curve
\addplot[thick, blue!70!black, domain=0:1, samples=50] ({x}, {1-x});

% GRPO point
\node[circle, fill=teal!70!black, inner sep=2.5pt] at (0, 1) {};
\node[teal!70!black, font=\small\bfseries, anchor=south west] at (0.04, 1.04) {GRPO ($\alpha\!=\!0$)};

% Dr. GRPO point
\node[circle, fill=orange!80!red, inner sep=2.5pt] at (1, 0) {};
\node[orange!80!red, font=\small\bfseries, anchor=south west] at (1.04, 0.04) {Dr.\ GRPO ($\alpha\!=\!1$)};

% Intermediate point
\node[circle, fill=blue!60!black, inner sep=2pt] at (0.5, 0.5) {};
\node[blue!60!black, font=\scriptsize, anchor=north west] at (0.53, 0.46) {$\alpha\!=\!0.5$};

% Curve label
\node[blue!60!black, font=\scriptsize, rotate=-35, anchor=north] at (0.25, 0.92) {Pareto frontier};
\end{axis}
\end{tikzpicture}
\caption{The impossibility tradeoff. The origin (zero bias on both axes) is unreachable. GRPO ($\alpha=0$) and Dr.\ GRPO ($\alpha=1$) occupy opposite ends of the Pareto frontier parameterized by $f_\alpha(L) = L^{\alpha-1}$.}
\label{fig:tradeoff}
\end{figure*}

\begin{proof}
    We show that P1 and P2 impose mutually contradictory constraints on $f$.

    \medskip
    \noindent\textbf{Step 1: If P1 holds over the policy class $\Pi$, then $f(L)$ must be a constant function.}

    Pick any $L_0\in\mathcal L$. By Assumption~\ref{ass:fixed_length_realizable}, there exists a policy $\pi^{(L_0)}\in\Pi$ under which the trajectory length equals $L_0$ almost surely given the prompt, while the token content remains random.
    Under this policy, for all $i$,
    \begin{equation}
        L_i=L_0,
    \end{equation}
    so the estimator can be written as
    \begin{equation}
        \hat g_f
        =
        f(L_0)\cdot
        \frac{1}{G}\sum_{i=1}^{G}\tilde A_i S_i.
    \end{equation}

    Expanding the baseline,
    \begin{equation}
        \tilde A_i
        =
        R_i-\frac{1}{G}\sum_{j=1}^{G}R_j
        =
        \left(1-\frac{1}{G}\right)R_i
        -\frac{1}{G}\sum_{j\neq i}R_j.
    \end{equation}
    Therefore,
    \begin{equation}
        \mathbb E_{\pi^{(L_0)}}[\tilde A_i S_i]
        =
        \left(1-\frac{1}{G}\right)\mathbb E_{\pi^{(L_0)}}[R_iS_i]
        -\frac{1}{G}\sum_{j\neq i}\mathbb E_{\pi^{(L_0)}}[R_jS_i].
    \end{equation}

    For $j\neq i$, since within-group trajectories are conditionally i.i.d., $R_j$ and $S_i$ are independent; moreover, by the score function identity,
    \begin{equation}
        \mathbb E_{\pi^{(L_0)}}[S_i]=0.
    \end{equation}
    Hence,
    \begin{equation}
        \mathbb E_{\pi^{(L_0)}}[R_jS_i]
        =
        \mathbb E_{\pi^{(L_0)}}[R_j]\,
        \mathbb E_{\pi^{(L_0)}}[S_i]
        =
        0.
    \end{equation}
    Thus,
    \begin{equation}
        \mathbb E_{\pi^{(L_0)}}[\tilde A_i S_i]
        =
        \left(1-\frac{1}{G}\right)\mathbb E_{\pi^{(L_0)}}[R_iS_i].
    \end{equation}

    By the REINFORCE identity,
    \begin{equation}
        \mathbb E_{\pi^{(L_0)}}[R_iS_i]
        =
        \nabla_\theta J(\pi^{(L_0)}),
    \end{equation}
    yielding
    \begin{equation}
        \mathbb E_{\pi^{(L_0)}}[\hat g_f]
        =
        f(L_0)\,\frac{G-1}{G}\,\nabla_\theta J(\pi^{(L_0)}).
    \end{equation}

    If P1 holds over the policy class $\Pi$, there exists a length-independent constant $c>0$ such that
    \begin{equation}
        \mathbb E_{\pi^{(L_0)}}[\hat g_f]
        =
        c\,\nabla_\theta J(\pi^{(L_0)}).
    \end{equation}
    Therefore,
    \begin{equation}
        f(L_0)\,\frac{G-1}{G}=c.
    \end{equation}
    Since $L_0$ is arbitrary in $\mathcal L$, $f(L)$ must be the same for all $L\in\mathcal L$.
    That is, there exists a constant $c_0>0$ such that
    \begin{equation}
        \label{eq:f_constant_policy_class}
        f(L)\equiv c_0,\qquad \forall L\in\mathcal L.
    \end{equation}

    \medskip
    \noindent\textbf{Step 2: If P2 holds, then under the theorem's assumptions $f$ cannot be a constant function.}

    By Definition~\ref{def:p2_uniform}, if the estimator $\hat g_f$ satisfies length neutrality P2 under the scale functional $\rho$, then for every policy $\pi\in\Pi$ and every effective training signal value $a$ for which the conditional expectation is defined, there exists a constant $C_{\pi,a}$ depending only on $(\pi,a)$ and not on the length $L$, such that for all realizable lengths $L\in\mathcal L$,
    \begin{equation}
        f(L)\,\Gamma_{\pi,\rho}(L;a)=C_{\pi,a}.
        \label{eq:p2_constant}
    \end{equation}

    Now fix the policy $\pi^\star\in\Pi$, the effective training signal value $a^\star$, and the two distinct lengths $L_1,L_2\in\mathcal L$ from the theorem's assumptions, satisfying
    \begin{equation}
        \Gamma_{\pi^\star,\rho}(L_1;a^\star)
        \neq
        \Gamma_{\pi^\star,\rho}(L_2;a^\star).
        \label{eq:gamma_separation}
    \end{equation}

    We show that $f$ cannot be a constant function.

    Suppose for contradiction that $f$ is constant, i.e., there exists a constant $c_0>0$ such that
    \begin{equation}
        f(L)\equiv c_0,\qquad \forall L\in\mathcal L.
        \label{eq:f_constant_step2}
    \end{equation}
    Substituting \eqref{eq:f_constant_step2} into \eqref{eq:p2_constant} with $\pi=\pi^\star$ and $a=a^\star$, we obtain for all realizable lengths $L\in\mathcal L$,
    \begin{equation}
        c_0\,\Gamma_{\pi^\star,\rho}(L;a^\star)=C_{\pi^\star,a^\star}.
        \label{eq:constant_gamma_relation}
    \end{equation}
    In particular, for $L_1$ and $L_2$,
    \begin{equation}
        c_0\,\Gamma_{\pi^\star,\rho}(L_1;a^\star)=C_{\pi^\star,a^\star},
    \end{equation}
    and
    \begin{equation}
        c_0\,\Gamma_{\pi^\star,\rho}(L_2;a^\star)=C_{\pi^\star,a^\star}.
    \end{equation}
    Since $c_0>0$, these two equations imply
    \begin{equation}
        \Gamma_{\pi^\star,\rho}(L_1;a^\star)
        =
        \Gamma_{\pi^\star,\rho}(L_2;a^\star),
    \end{equation}
    contradicting \eqref{eq:gamma_separation}.

    Therefore, under the theorem's assumptions, any weight function $f$ satisfying P2 cannot be a constant function.

    \medskip
    \noindent\textbf{Step 3: Contradiction.}

    Step 1 shows: if P1 holds over the policy class $\Pi$, then $f(L)$ must be a constant function.
    Step 2 shows: if P2 holds and there exists a policy for which $\Gamma_{\pi,\rho}(L;a)$ varies non-trivially with length, then $f(L)$ cannot be a constant function.

    These are contradictory.
    Therefore, under the theorem's assumptions, no weight function $f$ depending only on length can simultaneously satisfy P1 and P2 over the policy class $\Pi$.
\end{proof}

\paragraph{Illustrative example.}
Consider two trajectories for the same prompt with lengths $L_s\ll L_\ell$, compared under the same effective training signal.
If under some pre-specified scale functional $\rho$, the longer trajectory has a larger typical score-sum magnitude, i.e.,
\[
\Gamma_{\pi,\rho}(L_\ell;a)>\Gamma_{\pi,\rho}(L_s;a),
\]
then constant weights preserve this length-induced scale disparity,
while any length compensation attempting to eliminate this disparity must deviate from constant weights.
This example serves only to illustrate the structural conflict in the theorem and does not form part of the proof.

\paragraph{Scope of the theorem.}
Theorem~\ref{thm:policy_class_impossibility} does not claim that a specific functional form (e.g., $1/L$) is necessarily optimal; it merely states: when the typical score-sum magnitude under fixed effective training signal varies non-trivially with length, no unified weight function depending only on length can simultaneously satisfy P1 and P2.

Furthermore, the theorem only excludes weight functions that depend \emph{solely on length}; more general estimator designs, such as weighting schemes that depend on token position, context, score geometry, or finer-grained credit assignment, are not within the scope of this exclusion.

\begin{remark}[Essence of the conflict]
    \label{rem:core_tension_policy_class}
    P1 requires that a uniform length weight does not alter the original trajectory-level policy gradient objective;
    P2 requires that this weight compensates for the non-trivial variation of score-sum magnitude with length.
    When P1 constrains $f(L)$ to be a constant function while P2 demands it to vary with length, the two become structurally irreconcilable.
\end{remark}

%==========================================================================

\subsection{Examples}
For ease of understanding, see the supplementary material's ``Intuitive Examples of Asymmetric Length Behavior'' and ``Extreme Example'' sections. They show that GRPO's length bias manifests as correct responses tending to be shorter and incorrect responses tending to be longer, while Dr.\ GRPO's length bias manifests as both correct and incorrect responses tending to be longer.

%======================================================================
\section{Corollaries and Analysis}
\label{sec:corollaries}
%======================================================================

\subsection{Tradeoff Spectrum}

\begin{corollary}[Parametric Tradeoff Family]
\label{cor:spectrum}
Consider the parametric family $f_\alpha(L) = L^{\alpha - 1}$, $\alpha \in [0, 1]$:
\begin{itemize}[nosep,leftmargin=1.5em]
    \item $\alpha = 0$: $f_0(L) = 1/L$ \,---\, GRPO. Approximately satisfies P2 (length invariant) but violates P1 (biased gradient).
    \item $\alpha = 1$: $f_1(L) = 1$ \,---\, Dr.\ GRPO. Satisfies P1 (unbiased gradient) but violates P2 (length biased).
    \item $\alpha \in (0,1)$: intermediate tradeoff. Partially biased gradient, partially length-dependent.
\end{itemize}
Gradient estimation bias is proportional to $|\alpha - 1|$ and length bias is proportional to $\alpha$, establishing an inverse relationship.
\end{corollary}

\cref{fig:tradeoff} visualizes this tradeoff.

\subsection{Quantifying Dr.\ GRPO's Length Bias}

\begin{corollary}[Dr.\ GRPO's length bias]
\label{cor:dr_grpo_bias}
Under Dr.\ GRPO ($f(L) = 1$) with $G=2$ and binary outcome reward, let $\bo_1$ and $\bo_2$ be two trajectories with lengths $L_1$ and $L_2$.
Their advantages satisfy $|\Atilde_1| = |\Atilde_2| = 0.5$.
The effective gradient weight of trajectory $\bo_i$ is:
\begin{equation}
w_i = \frac{L_i}{L_1 + L_2}.
\end{equation}
For length ratio $r = L_{\max}/L_{\min}$, the longer trajectory captures:
\begin{equation}
w_{\mathrm{long}} = \frac{r}{1+r}
\end{equation}
of the total gradient magnitude, approaching $100\%$ as $r \to \infty$.
Under GRPO ($f(L) = 1/L$), $w_1 = w_2 = 0.5$, independent of length.
\end{corollary}

\begin{proof}
With $G=2$ and binary reward, exactly one trajectory is correct ($R=1$) and one incorrect ($R=0$), giving $\mathrm{mean}(\bR) = 0.5$ and $|\Atilde_1| = |\Atilde_2| = 0.5$.
Under Dr.\ GRPO, the gradient contribution magnitude of $\bo_i$ is proportional to $f(|\bo_i|) \cdot |\Atilde_i| \cdot |\bo_i| = 1 \cdot 0.5 \cdot L_i$.
The share is $w_i = L_i / (L_1 + L_2)$.
Under GRPO, the contribution is $(1/L_i) \cdot 0.5 \cdot L_i = 0.5$, independent of length.
\end{proof}

\begin{example}[Extreme case]
\label{ex:extreme}
Let $G=2$, $\bo_1$ correct ($R=1$, length 10 tokens), $\bo_2$ incorrect ($R=0$, length 10{,}000 tokens). The advantages are $\Atilde_1 = +0.5$, $\Atilde_2 = -0.5$.
Under Dr.\ GRPO, $\bo_2$'s gradient contribution is $10{,}000 \times 0.5 = 5{,}000$, while $\bo_1$'s is only $10 \times 0.5 = 5$.
The longer trajectory captures $\frac{5000}{5005} = 99.9\%$ of the gradient signal, nearly completely drowning out the reinforcement of the correct answer.
Under GRPO, both contribute 50\%.
A step-by-step derivation of this example (including gradient decomposition and its effect on parameter updates) is provided in the supplementary material's ``Extreme Example'' section.
\end{example}

\cref{tab:bias} shows the severity of this effect at various length ratios.

\begin{table}[t]
\centering
\caption{Gradient weight shares of Dr.\ GRPO vs.\ GRPO at different length ratios ($G=2$, binary reward). Under GRPO, both trajectories always receive equal weight.}
\label{tab:bias}
\scriptsize
\setlength{\tabcolsep}{3pt}
\begin{tabular}{@{}ccccc@{}}
\toprule
& \multicolumn{2}{c}{Dr.\ GRPO} & \multicolumn{2}{c}{GRPO} \\
\cmidrule(lr){2-3}\cmidrule(l){4-5}
Length ratio $r$ & $w_{\text{long}}$ & $w_{\text{short}}$ & $w_{\text{long}}$ & $w_{\text{short}}$ \\
\midrule
1:1 & 50.0\% & 50.0\% & 50.0\% & 50.0\% \\
2:1 & 66.7\% & 33.3\% & 50.0\% & 50.0\% \\
5:1 & 83.3\% & 16.7\% & 50.0\% & 50.0\% \\
10:1 & 90.9\% & 9.1\% & 50.0\% & 50.0\% \\
50:1 & 98.0\% & 2.0\% & 50.0\% & 50.0\% \\
100:1 & 99.0\% & 1.0\% & 50.0\% & 50.0\% \\
\bottomrule
\end{tabular}
\end{table}

\begin{example}[Practical relevance]
\label{ex:practical}
\citet{liu2025understanding} reported in their Table~5 that DeepSeek-R1-Zero produces correct answers averaging 4{,}965 tokens and incorrect answers averaging 8{,}206 tokens (a ratio of approximately $1{:}1.65$).
Under Dr.\ GRPO with $G=2$, the incorrect (longer) trajectory would capture approximately $\frac{8206}{4965 + 8206} \approx 62.3\%$ of the gradient, deviating 24.6 percentage points from the balanced 50\%.
While this proportion may appear moderate for a single update, the bias accumulates over hundreds of training iterations, systematically favoring longer responses.
\end{example}

\subsection{Quantifying GRPO's Gradient Bias}

For completeness, we also characterize the gradient bias introduced by GRPO.

\begin{corollary}[GRPO's gradient bias]
\label{cor:grpo_bias}
Under GRPO ($f(L) = 1/L$), the gradient estimator satisfies:
\begin{equation}
\E[\hat{g}_{1/L}] - \nabla_\theta J = \E\left[\frac{1}{G}\sum_{i} \Atilde_i \left(\frac{1}{|\bo_i|} - 1\right) \nabla_\theta\log\pith(\bo_i|\bq)\right].
\end{equation}
This bias is non-zero when the trajectory length $|\bo_i|$ is correlated with the score function $\nabla_\theta \log\pith(\bo_i|\bq)$. This is generally always the case since the policy determines when the EOS token is generated.
\end{corollary}

\begin{proof}
By direct computation:
$\E[\hat{g}_{1/L}] = \E\left[\frac{1}{G}\sum_i \Atilde_i \frac{1}{|\bo_i|}\nabla_\theta\log\pith(\bo_i|\bq)\right]$ and $\nabla_\theta J = \E\left[\frac{1}{G}\sum_i \Atilde_i \nabla_\theta\log\pith(\bo_i|\bq)\right]$.
The difference follows directly by linearity.
Since $|\bo_i|$ is determined by when $\pith$ generates the EOS token, $|\bo_i|$ and $\nabla_\theta\log\pith(\bo_i|\bq)$ are dependent, making the bias generally non-zero.
\end{proof}

\subsection{Extension to General Group Size}

\begin{corollary}[General $G$ + binary reward]
\label{cor:general_G}
For group size $G$ with binary reward, if $K$ out of $G$ responses are correct, the advantages are $\Atilde_{\mathrm{correct}} = 1 - K/G$ and $\Atilde_{\mathrm{incorrect}} = -K/G$.
Under Dr.\ GRPO, the effective weight of trajectory $\bo_i$ is still proportional to $|\bo_i| \cdot |\Atilde_i|$.
The length bias exists for all $G$: longer trajectories always contribute more to the gradient, regardless of their correctness:
\begin{equation}
w_i = \frac{|\bo_i| \cdot |\Atilde_i|}{\sum_{j=1}^{G} |\bo_j| \cdot |\Atilde_j|}.
\end{equation}
\end{corollary}

%======================================================================
\section{Discussion}
\label{sec:discussion}
%======================================================================

\paragraph{``Done Right'' is a misnomer.}
Dr.\ GRPO~\citep{liu2025understanding}, titled ``Understanding R1-Zero-Like Training: A Critical Perspective,'' positions its contribution as fixing GRPO's optimization biases.
The phrase ``GRPO Done Right'' implies a single correct formulation.
Our impossibility theorem (\cref{thm:policy_class_impossibility}) shows this is not the case: GRPO and Dr.\ GRPO navigate different points on the inherent tradeoff between gradient unbiasedness and length invariance.
Calling one of them ``done right'' obscures the fact that both make legitimate but different tradeoff choices.

\paragraph{When does the tradeoff matter?}
The practical importance of the tradeoff depends on the variance of response lengths.
When all responses to a given prompt have similar lengths (e.g., simple arithmetic), the difference between $\alpha = 0$ and $\alpha = 1$ is negligible.
When response lengths vary substantially, the choice of $\alpha$ materially affects training dynamics. This situation is typical in reasoning tasks: correct solutions may be concise while incorrect attempts tend to be verbose~\citep{deepseekai2026deepseekr1incentivizingreasoningcapability}.

\paragraph{Practical guidance.}
While we do not propose a specific algorithm, our analysis suggests:
(i)~When response length variance is high, a smaller $\alpha$ (closer to GRPO) may be preferable to prevent longer trajectories from dominating the gradient.
(ii)~When gradient bias is the primary concern (e.g., early in training when the policy changes rapidly), a larger $\alpha$ (closer to Dr.\ GRPO) provides more accurate gradient estimates.
(iii)~The optimal $\alpha$ may vary across training phases, suggesting that a curriculum approach could be beneficial.

\paragraph{Implications for training dynamics.}
A practical implication of \cref{cor:dr_grpo_bias} deserves attention: when long correct responses receive $L$ times more reinforcement signal than short correct responses, the policy may gradually shift toward generating longer outputs.
The complete causal chain from gradient dominance to behavioral change also involves clipping, learning rate, and multi-step optimization, which lie beyond the scope of our single-step analysis. However, the systematic asymmetry in gradient signals provides a necessary condition for this trend.
Conversely, under GRPO ($\alpha=0$), a 10-token short correct response and a 10{,}000-token long correct response receive the same total reinforcement signal. This provides no incentive at the gradient level to favor longer or shorter outputs.

\paragraph{Relationship to other biases.}
Our analysis complements \citet{yang2026grouprelativeadvantagebiased}. The latter studies a different bias in GRPO: \emph{difficulty bias}. This bias refers to the group-relative advantage estimator systematically underestimating advantages for difficult prompts and overestimating them for easy prompts.
The length bias we identify is orthogonal, arising from within-group length variation rather than between-group difficulty variation.
The standard deviation normalization in GRPO contributes to difficulty bias~\citep{liu2025understanding}; our impossibility result is independent of whether $\mathrm{std}(\bR)$ normalization is used.

\paragraph{Limitations.}
Our impossibility result is specific to the \emph{outcome reward} setting, where each trajectory is assigned a scalar reward broadcast to all tokens.
Under \emph{process reward}~\citep{schulman2018highdimensionalcontinuouscontrolusing}, different tokens receive different advantage estimates and the problem structure changes. The advantage is no longer constant across tokens, and the $\sum_t$ aggregation is no longer simply $|\bo_i| \cdot \Atilde_i$.
Extending the impossibility analysis to process rewards is an interesting future direction.
Furthermore, our analysis focuses on single-step gradient estimators. The interaction between length bias and multi-step optimization dynamics (e.g., through PPO-style clipping) warrants further investigation.

%======================================================================
\section{Conclusion}
\label{sec:conclusion}
%======================================================================

We have established a fundamental impossibility result for group-based policy optimization under outcome rewards: gradient unbiasedness and length invariance cannot coexist.
This reveals that GRPO and Dr.\ GRPO are not in a ``biased'' vs.\ ``correct'' relationship, but instead represent two principled tradeoff choices on the Pareto frontier.
We hope this clarification helps the community make more informed algorithmic decisions, recognizing that the appropriate operating point depends on the specific characteristics of the training setting, especially the distribution of response lengths.

%======================================================================

\begin{thebibliography}{11}
\providecommand{\natexlab}[1]{#1}

\bibitem[{Ahmadian et~al.(2024)Ahmadian, Cremer, Gall{\'e}, Fadaee, Kreutzer,
  Pietquin, {\"U}st{\"u}n, and Hooker}]{ahmadian-etal-2024-back}
Ahmadian, A.; Cremer, C.; Gall{\'e}, M.; Fadaee, M.; Kreutzer, J.; Pietquin,
  O.; {\"U}st{\"u}n, A.; and Hooker, S. 2024.
\newblock Back to Basics: Revisiting {REINFORCE}-Style Optimization for
  Learning from Human Feedback in {LLM}s.
\newblock In Ku, L.-W.; Martins, A.; and Srikumar, V., eds., \emph{Proceedings
  of the 62nd Annual Meeting of the Association for Computational Linguistics
  (Volume 1: Long Papers)}, 12248--12267. Bangkok, Thailand: Association for
  Computational Linguistics.

\bibitem[{DeepSeek-AI et~al.(2026)DeepSeek-AI, Guo, Yang, Zhang, Song, Wang,
  Zhu, Xu, Zhang, Ma, Bi, Zhang, Yu, Wu, Wu, Gou, Shao, Li, Gao, Liu, Xue,
  Wang, Wu, Feng, Lu, Zhao, Deng, Zhang, Ruan, Dai, Chen, Ji, Li, Lin, Dai,
  Luo, Hao, Chen, Li, Zhang, Bao, Xu, Wang, Ding, Xin, Gao, Qu, Li, Guo, Li,
  Wang, Chen, Yuan, Qiu, Li, Cai, Ni, Liang, Chen, Dong, Hu, Gao, Guan, Huang,
  Yu, Wang, Zhang, Zhao, Wang, Zhang, Xu, Xia, Zhang, Zhang, Tang, Li, Wang,
  Li, Tian, Huang, Zhang, Wang, Chen, Du, Ge, Zhang, Pan, Wang, Chen, Jin,
  Chen, Lu, Zhou, Chen, Ye, Wang, Yu, Zhou, Pan, Li, Zhou, Wu, Ye, Yun, Pei,
  Sun, Wang, Zeng, Zhao, Liu, Liang, Gao, Yu, Zhang, Xiao, An, Liu, Wang, Chen,
  Nie, Cheng, Liu, Xie, Liu, Yang, Li, Su, Lin, Li, Jin, Shen, Chen, Sun, Wang,
  Song, Zhou, Wang, Shan, Li, Wang, Wei, Zhang, Xu, Li, Zhao, Sun, Wang, Yu,
  Zhang, Shi, Xiong, He, Piao, Wang, Tan, Ma, Liu, Guo, Ou, Wang, Gong, Zou,
  He, Xiong, Luo, You, Liu, Zhou, Zhu, Xu, Huang, Li, Zheng, Zhu, Ma, Tang,
  Zha, Yan, Ren, Ren, Sha, Fu, Xu, Xie, Zhang, Hao, Ma, Yan, Wu, Gu, Zhu, Liu,
  Li, Xie, Song, Pan, Huang, Xu, Zhang, and
  Zhang}]{deepseekai2026deepseekr1incentivizingreasoningcapability}
DeepSeek-AI; Guo, D.; Yang, D.; Zhang, H.; Song, J.; Wang, P.; Zhu, Q.; Xu, R.;
  Zhang, R.; Ma, S.; Bi, X.; Zhang, X.; Yu, X.; Wu, Y.; Wu, Z.~F.; Gou, Z.;
  Shao, Z.; Li, Z.; Gao, Z.; Liu, A.; Xue, B.; Wang, B.; Wu, B.; Feng, B.; Lu,
  C.; Zhao, C.; Deng, C.; Zhang, C.; Ruan, C.; Dai, D.; Chen, D.; Ji, D.; Li,
  E.; Lin, F.; Dai, F.; Luo, F.; Hao, G.; Chen, G.; Li, G.; Zhang, H.; Bao, H.;
  Xu, H.; Wang, H.; Ding, H.; Xin, H.; Gao, H.; Qu, H.; Li, H.; Guo, J.; Li,
  J.; Wang, J.; Chen, J.; Yuan, J.; Qiu, J.; Li, J.; Cai, J.~L.; Ni, J.; Liang,
  J.; Chen, J.; Dong, K.; Hu, K.; Gao, K.; Guan, K.; Huang, K.; Yu, K.; Wang,
  L.; Zhang, L.; Zhao, L.; Wang, L.; Zhang, L.; Xu, L.; Xia, L.; Zhang, M.;
  Zhang, M.; Tang, M.; Li, M.; Wang, M.; Li, M.; Tian, N.; Huang, P.; Zhang,
  P.; Wang, Q.; Chen, Q.; Du, Q.; Ge, R.; Zhang, R.; Pan, R.; Wang, R.; Chen,
  R.~J.; Jin, R.~L.; Chen, R.; Lu, S.; Zhou, S.; Chen, S.; Ye, S.; Wang, S.;
  Yu, S.; Zhou, S.; Pan, S.; Li, S.~S.; Zhou, S.; Wu, S.; Ye, S.; Yun, T.; Pei,
  T.; Sun, T.; Wang, T.; Zeng, W.; Zhao, W.; Liu, W.; Liang, W.; Gao, W.; Yu,
  W.; Zhang, W.; Xiao, W.~L.; An, W.; Liu, X.; Wang, X.; Chen, X.; Nie, X.;
  Cheng, X.; Liu, X.; Xie, X.; Liu, X.; Yang, X.; Li, X.; Su, X.; Lin, X.; Li,
  X.~Q.; Jin, X.; Shen, X.; Chen, X.; Sun, X.; Wang, X.; Song, X.; Zhou, X.;
  Wang, X.; Shan, X.; Li, Y.~K.; Wang, Y.~Q.; Wei, Y.~X.; Zhang, Y.; Xu, Y.;
  Li, Y.; Zhao, Y.; Sun, Y.; Wang, Y.; Yu, Y.; Zhang, Y.; Shi, Y.; Xiong, Y.;
  He, Y.; Piao, Y.; Wang, Y.; Tan, Y.; Ma, Y.; Liu, Y.; Guo, Y.; Ou, Y.; Wang,
  Y.; Gong, Y.; Zou, Y.; He, Y.; Xiong, Y.; Luo, Y.; You, Y.; Liu, Y.; Zhou,
  Y.; Zhu, Y.~X.; Xu, Y.; Huang, Y.; Li, Y.; Zheng, Y.; Zhu, Y.; Ma, Y.; Tang,
  Y.; Zha, Y.; Yan, Y.; Ren, Z.~Z.; Ren, Z.; Sha, Z.; Fu, Z.; Xu, Z.; Xie, Z.;
  Zhang, Z.; Hao, Z.; Ma, Z.; Yan, Z.; Wu, Z.; Gu, Z.; Zhu, Z.; Liu, Z.; Li,
  Z.; Xie, Z.; Song, Z.; Pan, Z.; Huang, Z.; Xu, Z.; Zhang, Z.; and Zhang, Z.
  2026.
\newblock DeepSeek-R1: Incentivizing Reasoning Capability in LLMs via
  Reinforcement Learning.
\newblock arXiv:2501.12948.

\bibitem[{Hu et~al.(2025)Hu, Zhang, Han, Jiang, Zhang, and
  Shum}]{hu2025openreasonerzeroopensourceapproach}
Hu, J.; Zhang, Y.; Han, Q.; Jiang, D.; Zhang, X.; and Shum, H.-Y. 2025.
\newblock Open-Reasoner-Zero: An Open Source Approach to Scaling Up
  Reinforcement Learning on the Base Model.
\newblock arXiv:2503.24290.

\bibitem[{Kool, van Hoof, and Welling(2019)}]{kool2019buy}
Kool, W.; van Hoof, H.; and Welling, M. 2019.
\newblock Buy 4 {REINFORCE} Samples, Get a Baseline for Free!

\bibitem[{Liu et~al.(2025)Liu, Chen, Li, Qi, Pang, Du, Lee, and
  Lin}]{liu2025understanding}
Liu, Z.; Chen, C.; Li, W.; Qi, P.; Pang, T.; Du, C.; Lee, W.~S.; and Lin, M.
  2025.
\newblock Understanding R1-Zero-Like Training: A Critical Perspective.
\newblock In \emph{Second Conference on Language Modeling}.

\bibitem[{Schulman et~al.(2018)Schulman, Moritz, Levine, Jordan, and
  Abbeel}]{schulman2018highdimensionalcontinuouscontrolusing}
Schulman, J.; Moritz, P.; Levine, S.; Jordan, M.; and Abbeel, P. 2018.
\newblock High-Dimensional Continuous Control Using Generalized Advantage
  Estimation.
\newblock arXiv:1506.02438.

\bibitem[{Shao et~al.(2024)Shao, Wang, Zhu, Xu, Song, Bi, Zhang, Zhang, Li, Wu,
  and Guo}]{shao2024deepseekmathpushinglimitsmathematical}
Shao, Z.; Wang, P.; Zhu, Q.; Xu, R.; Song, J.; Bi, X.; Zhang, H.; Zhang, M.;
  Li, Y.~K.; Wu, Y.; and Guo, D. 2024.
\newblock DeepSeekMath: Pushing the Limits of Mathematical Reasoning in Open
  Language Models.
\newblock arXiv:2402.03300.

\bibitem[{Sutton and Barto(2018)}]{sutton2018reinforcement}
Sutton, R.~S.; and Barto, A.~G. 2018.
\newblock \emph{Reinforcement Learning: An Introduction}.
\newblock The MIT Press, 2 edition.

\bibitem[{Williams(1992)}]{Williams_1992}
Williams, R.~J. 1992.
\newblock Simple statistical gradient-following algorithms for connectionist
  reinforcement learning.
\newblock \emph{Machine Learning}, 8(3-4): 229–256.

\bibitem[{Yang et~al.(2026)Yang, Chen, Wang, Lu, Chai, Yin, Lin, Ma, Zhuang,
  Wang, Yang, Li, and Ban}]{yang2026grouprelativeadvantagebiased}
Yang, F.; Chen, Z.; Wang, X.; Lu, X.; Chai, J.; Yin, G.; Lin, W.; Ma, S.;
  Zhuang, F.; Wang, D.; Yang, Y.; Li, J.; and Ban, Y. 2026.
\newblock Your Group-Relative Advantage Is Biased.
\newblock arXiv:2601.08521.

\bibitem[{Zeng et~al.(2025)Zeng, Huang, Liu, He, Liu, Ma, and
  He}]{zeng2025simplerl}
Zeng, W.; Huang, Y.; Liu, W.; He, K.; Liu, Q.; Ma, Z.; and He, J. 2025.
\newblock 7B Model and 8K Examples: Emerging Reasoning with Reinforcement
  Learning is Both Effective and Efficient.
\newblock \url{https://hkust-nlp.notion.site/simplerl-reason}.
\newblock Notion Blog.

\end{thebibliography}
\end{document}